\documentclass[twoside,11pt]{article}

\usepackage[utf8]{inputenc}
\usepackage{amsmath,amsthm,amssymb}

\usepackage{ALT/jmlr2e}
\usepackage[ruled,vlined]{algorithm2e}
\usepackage{algpseudocode}%algorithm}%,}
\usepackage{footnote}           
\makesavenoteenv{tabular}

\newtheorem{theorem}{Theorem}
\newtheorem{lemma}[theorem]{Lemma} 
\newtheorem{lemma_thm}{Lemma}[theorem]
\newtheorem{corollary}[theorem]{Corollary}
\newtheorem{claim}{Claim}[section]
\newtheorem*{definition*}{Definition}

\newcommand{\norm}[1]{\left\lVert#1\right\rVert}

\firstpageno{1}

\begin{document}

\title{Top-k Combinatorial Bandits with Full-Bandit Feedback}

\author{\name Idan Rejwan \email idanrejwan@mail.tau.ac.il \\
       \addr Blavatnik School of Computer Science, Tel Aviv University, Tel Aviv Israel
       \AND
       \name Yishay Mansour \email mansour.yishay@gmail.com \\
       \addr Blavatnik School of Computer Science, Tel Aviv University, Tel Aviv Israel\\
and Google Research, Tel Aviv}

\editor{PLACEHOLDER}

\maketitle

\begin{abstract}%   <- trailing '%' for backward compatibility of .sty file
\emph{Top-k Combinatorial Bandits} generalize multi-armed bandits, where at each round any subset of $k$ out of $n$ arms may be chosen and the sum of the rewards is gained.
We address the full-bandit feedback, in which the agent observes only the sum of rewards, in contrast to the semi-bandit feedback, in which the agent observes also the individual arms' rewards.
We present the \emph{Combinatorial Successive Accepts and Rejects} (CSAR) algorithm, which generalizes SAR \citep{bubeck2013multiple} for top-k combinatorial bandits.
Our main contribution is an efficient sampling scheme that uses Hadamard matrices in order to estimate accurately the individual arms' expected rewards.
We discuss two variants of the algorithm, the first minimizes the sample complexity and the second minimizes the regret. 
We also prove a lower bound on sample complexity, which is tight for $k=O(1)$.
Finally, we run experiments and show that our algorithm outperforms other methods.
\end{abstract}

\begin{keywords}
  Multi-Armed Bandits, Combinatorial Bandits, Top-k Bandits, Hadamard Matrix, Sample Complexity, Regret Minimization
\end{keywords}

\section{Introduction}\label{sc:intro}

Multi-armed bandit (MAB) is an extensively studied problem in statistics and machine learning. The classical version of this problem is formulated as a system of $n$ arms (or actions), each having an unknown distribution of rewards. An agent repeatedly plays these arms in order to find the best arm and maximize its reward \citep{robbins1952some}. 

The MAB research focuses on two different objectives. The first aims to maximize the reward accumulated by the agent while playing the arms. This objective highlights the trade-off between exploration and exploitation, i.e., the balance between staying with the arm that gave highest reward in the past and exploring new arms that might give higher reward in the future. Success in this goal is measured by \emph{regret}, which is the difference between the best arm's expected reward over the time horizon and the reward accumulated by the agent over the same time. The second objective, sometimes referred to as \emph{best arm identification} or \emph{pure exploration}, aims to minimize the \emph{sample complexity} which is the number of steps until identifying the best arm with high probability. These two objectives might contradict each other, meaning that a policy which is good for finding the best arm quickly is not necessarily good for accumulating high reward \citep{bubeck2009pure}.

An extension of the standard MAB model is the \emph{Combinatorial Bandits} model \citep{cesa2012combinatorial,chen2013combinatorial}. In this model, instead of choosing one arm at each round, a decision set of actions is given, where each action is a subset of arms. \emph{Top-k} is a special case of combinatorial bandits, in which the decision set includes all subsets of size $k$ out of $n$ arms, and each action's reward is the sum of the $k$ arms. Combinatorial Bandits have two variants, depending on the feedback observed by the agent. In the simpler one the agent observes in each round the rewards of each of the $k$ individual arms, in addition to the aggregated reward. Such model is referred to as the \emph{semi-bandit} feedback. This is in contrast to the \emph{full-bandit} feedback, where the only feedback observed by the agent is the aggregated reward. Although much of the research studies the semi-bandit feedback \citep{chen2013combinatorial,chen2016combinatorial,combes2015combinatorial,kveton2015tight}, in many real-life problems it is costly or even impossible to gain information on each individual arm by itself. This is the case in, for example, crowd sourcing \citep{lin2014combinatorial} and adaptive routing \citep{awerbuch2004adaptive}, and also in scenarios where data privacy considerations come into play, such as online advertisement and medical trials.

Full-bandit feedback is harder than semi-bandit feedback, due to the lack of information about each individual arm. Each time a subset is sampled and an aggregated reward is observed, it is hard to assign the credit between the individual arms. One naive attempt to deal with it is to treat every possible subset as a distinct arm, and consider it as a classical MAB problem with $\binom{n}{k}$ arms. However, the number of arms is exponential, hence this approach is clearly inefficient. Additionally, it ignores the combinatorial structure that could extract some shared information between different subsets. Another attempt is to treat it as a special case of \emph{Linear Bandits}. In this model, each arm $a$ is a vector in a decision set $\mathcal{D}\subseteq\mathbb{R}^n$, and its expected reward is the inner product between $a\in\mathcal{D}$ and the reward vector $\theta\in\mathbb{R}^n$. Combinatorial bandits are actually a special case of linear bandits, where the decision set is limited to binary vectors with exactly $k$ ones. One could hope to use LinUCB, the highly established algorithm for linear bandits \citep{abbasi2011improved,dani2008stochastic,chu2011contextual}, to solve combinatorial bandits. This algorithm involves an optimization problem to find which subset to sample at each round, however for combinatorial decision sets the optimization is NP-hard \citep{dani2008stochastic, kuroki2019polynomial}.
Thus, we wish to find an algorithm that is (a) informative - gives enough information on each individual arm; (b) efficient - uses a small number of samples; and (c) polynomial time computable. Our main contribution is an algorithm that fulfills all three requirements, as we show theoretically and empirically. 

In this work, we describe an algorithm for full-bandit feedback that finds the optimal subset of arms efficiently. The algorithm is based on the \emph{Successive Accepts and Rejects (SAR)} algorithm \citep{bubeck2013multiple}, that iteratively estimates the arms within increasing level of accuracy, and accepts or rejects arms until it finds the optimal subset. While the original algorithm is designed for classical MABs, it is not clear how to estimate the expected rewards of the individual arms given full-bandit feedback. Our main novelty is thus describing an efficient method for estimating the individual arms' rewards and by this generalizing SAR to full-bandit feedback. We present a sampling scheme that uses Hadamard matrices to estimate the arms using a small number of samples. We show that this scheme is efficient, by proving that the number of samples needed to find the optimal subset with probability at least $1-\delta$ is at most $O\big(\sum_{i=1}^n \frac{1}{\Delta_i^2}\log\frac{n}{\delta}\big)$, where $\Delta_i$'s are the gaps between the optimal and sub-optimal arms (see Section \ref{sc:prelim} for formal definition). We also prove a lower bound of $\Omega\big(\frac{n}{\epsilon^2}\big)$ samples for finding a subset whose expected reward is within  $\epsilon$ of the optimal. Note that in the combinatorial model the feedback depends on $k$ actions, rather than a single one, thus it might be more informative.
%
% To the best of our knowledge, the algorithm we propose is the first algorithm for the full-bandit feedback that achieves the optimal sample complexity.
Second, we discuss regret minimization. We show that the algorithm that minimizes sample complexity does not minimize the regret. Instead, we suggest a modification to the algorithm that achieves $O\big(\frac{nk}{\Delta}\log T\big)$ distribution-dependent and $O(k\sqrt{nT})$ distribution-independent regret where $\Delta=\min \Delta_i$ and $T$ is the time horizon. To the best of our knowledge, this is the first algorithm to achieve $O(\log T)$ distribution-dependent regret in the full-bandit setting. %We also prove a lower bound of $\Omega(\sqrt{knT})$ regret.
%, but conjecture that there is a tighter bound according to similar models.
Finally, we conduct experiments that show that the proposed algorithm achieves small sample complexity and  regret comparing to other methods.

\subsection{Related Work}\label{sc:relatedWork}

\paragraph{Best Arm Identification.}
The problem of \emph{Best Arm Identification}, a.k.a. \emph{Pure Exploration}, was introduced by \cite{even2006action}, and later by \cite{bubeck2009pure}, where the goal is to find the best arm using a minimal number of samples. \cite{even2006action} describe two algorithms for this end, one of them is \emph{Successive Elimination} that in each round estimates all the arms with an increasing level of accuracy, and eliminates the arms which are far from the optimal arm with high confidence. This algorithm uses $O\big(\sum_{i=1}^n \frac{1}{\Delta_i^2}\log\frac{n}{\delta}\big)$ to find the optimal arm with probability at least $1-\delta$. It is the conceptual basis for a number of algorithms, including the one we describe in this work.

\paragraph{Multiple Arms Identification.}
As an extension for Best-Arm Identification, the goal of \emph{Multiple Arms Identification} is to find the best $k$ arms where the samples are still of one arm in each round.
This problem, a.k.a. \emph{Subset Selection} or \emph{Explore-k}, was introduced by \cite{kalyanakrishnan2010efficient}, and a variety of algorithms were designed for this end \citep{kalyanakrishnan2010efficient, kalyanakrishnan2012pac, chen2014combinatorial, zhou2014optimal}. One notable algorithm is \emph{Successive Accepts and Rejects (SAR)} \citep{bubeck2013multiple}, which generalizes Successive Elimination algorithm to multiple arms identification by adding a set of accepted arms that have been identified as part of the optimal arms.

\paragraph{Combinatorial Bandits.} 
Most of the works in the framework of stochastic combinatorial bandits address the semi-bandit feedback \citep{chen2013combinatorial,chen2016combinatorial,combes2015combinatorial,kveton2015tight,merlis2019batch}. For full-bandit feedback, only a few algorithms were suggested. One of them is ConfidenceBall$_1$ in \cite{dani2008stochastic} which is a polynomial time approximation for LinUCB for linear bandits with NP-hard decision sets. Another approximation for LinUCB is described by \cite{kuroki2019polynomial}, which uses an approximated method for quadratic optimization based on graphs. A different approach is taken by \cite{agarwal2018regret}, which is designed for cases when the aggregated reward is not necessarily the sum of individual arms. This algorithm is based on Explore-then-Exploit approach and achieves regret of $O(k^\frac{1}{2}n^\frac{1}{3}T^\frac{2}{3})$. 
%
%Another algorithm that should be mentioned is in \cite{lin2014combinatorial}. They consider a problem that somewhat generalizes the full-bandit setting, where the reward is not necessarily the sum of individual arms, but the feedback for the agent is a linear combination of the arms' rewards.
%
\cite{lin2014combinatorial} consider a problem that somewhat generalizes the full-bandit setting, where the reward is not necessarily the sum of individual arms, but the feedback for the agent is a linear combination of the arms' rewards and show an $O(T^{2/3}\log T)$ regret bound.
For the sake of completeness, we note that there are also a number of works on full-bandit feedback in the adversarial setting \citep{cesa2012combinatorial,combes2015combinatorial}.

\paragraph{Lower Bounds.}
%Alongside algorithms for the various models, a considerable part of the research is concerned with the question what is the minimal sample complexity and regret for any algorithm to succeed in the identification task. 
For best arm identification, %at least
$\Theta\big(\frac{n}{\epsilon^2}\log\frac{1}{\delta}\big)$ samples are necessary and sufficient for any $(\epsilon,\delta)$-PAC algorithm to identify the best arm \citep{mannor2004sample,even2006action}. For multiple arms identification, a slightly more samples are needed, where the lower bound is $\Omega\big(\frac{n}{\epsilon^2}\log\frac{k}{\delta}\big)$ \citep{kalyanakrishnan2012pac,kaufmann2013information}. Our work extends this bounds to the full-bandit feedback, providing a lower bound of $\Omega\big(\frac{n}{\epsilon^2}\big)$ on the sample complexity. \\
As for the regret, a seminal work by \cite{lai1985asymptotically} bounds the regret of classical MAB as $\Omega\Big(\sum_i \frac{1}{\Delta_i}\log T\Big)$. This result extents to $\Omega\big(c(\theta)\log T\big)$ for combinatorial bandits, where $c(\theta)$ is a solution of an optimization problems that depends on the distribution of rewards \citep{talebi2017stochastic}.\\
Another type of bounds is the distribution-independent regret bounds, that does not depend on the distribution of the arms rewards. For classical MABs, a well-known lower bound of $\Omega(\sqrt{nT})$ was proven by \cite{auer2002nonstochastic}. This result extends to $\Omega(\sqrt{knT})$ for combinatorial bandits with semi-bandit feedback \citep{kveton2015tight, lattimore2018toprank}. For full-bandit feedback, there are even stronger results of $\Omega(k\sqrt{nT})$ \citep{audibert2013regret} and $\Omega(k\sqrt{knT})$ \citep{cohen2017tight}, if the decision set is limited, i.e., not all subsets can be selected by the agent, which is not the case in our setting.
Another relevant bound is for linear bandits, where the regret is bounded by $\Omega(n\sqrt{T})$ \citep{dani2008stochastic}.
\section{Preliminaries}\label{sc:prelim}
% In this section, we formally define some preliminaries for the paper.\\ 

% \subsection{Combinatorial Bandits Model}
Suppose that there are $n$ arms numbered $1,2,\dots,n$, and each arm $i\in[n]$ is associated with a random variable $X_i=\theta_i+\eta_i$ such that $\theta_i$ is the expected reward and $\eta_i$ is 1-subgaussian noise. We assume the arms are ordered such that $\theta_1\geq\dots\geq\theta_n$, but this order is not known to the agent. In each round $t$, the agent selects a subset $S_t$ of $k$ arms and observes a reward $r_t=\sum_{i\in S_t} X_i$, where each arm $X_i$ is sampled independently.

The agent's objective is to find a subset $S$ that maximizes the expected reward $\mu_S=\mathbb{E}[r_S]$. Since the arms are independent we can write $\mu_S=\sum_{i\in S}\theta_i$. Accordingly, the optimal subset is $S^*=\{1,\dots,k\}$ with expected reward $\mu^*=\sum_{i=1}^k \theta_i$.

We adopt the $(\epsilon,\delta)$-PAC framework \citep{valiant1984theory}, in which the goal of the agent is to output a subset $S$ such that for any $\epsilon,\delta>0$, $Pr[\mu^*-\mu_S>\epsilon]<\delta$.

The regret of the agent over time horizon $T$ is defined as
\[R=\mathop{\mathbb{E}}\bigg[\sum_{t=1}^{T}\Big(\mu^*-r_t\Big)\bigg]=T\mu^*-\sum_{t=1}^{T}\mu_t\]
where $\mu_t=\mathbb{E}[r_t]$ is the expected reward at round $t$. The regret is measured in terms of the gaps between the arms. For every arm $i\in[n]$ we define the gap 
\[\Delta_i=\begin{cases}
 \theta_i-\theta_{k+1} & i\leq k\\
 \theta_k-\theta_i & i>k
\end{cases}\]
Note that the gaps are defined differently than for classical MAB, as the optimal arms also have gaps comparing to the best sub-optimal arm $k+1$. Intuitively, the gaps are the arms' distances from changing their status from optimal to sub-optimal arms and vice versa.\\
Finally, we define $\Delta=\min_i \Delta_i=\Delta_k=\Delta_{k+1}$.

\subsection{Hadamard Matrices}
The algorithm we present in this paper uses the Hadamard matrix. We define it here and discuss a few properties of it, for more information see \cite{horadam2012hadamard}.

\begin{definition*}
A square matrix $H$ of size $n$ is called Hadamard if its entries are $\pm 1$ and it satisfies $H^\intercal H=n\mathbb{I}$, where $\mathbb{I}$ is the identity matrix.
\end{definition*}

Hadamard matrices satisfy the following properties:
\begin{itemize}
    \item Any $H$ can be normalized such that the first row contains only positive entries.
    \item For any $i>1$, the $i^{th}$ row in $H$ contains an equal number of $+1$ and $-1$.
    \item For any $n$, there exists a Hadamard matrix of size $2^n$. It is conjectured that Hadamard matrices exist for any multiple of $4$, and the matrices for most of the multiples of $4$ up to $2000$ are known \citep{djokovic2008hadamard}.
\end{itemize}

It is interesting to mention that Hadamard matrices maximize the determinant of $\pm 1$ matrices, which makes them D-optimal design matrices \citep[see][chap.~4]{horadam2012hadamard}.
\section{Combinatorial Successive Accepts and Rejects Algorithm}
In this chapter we present the Combinatorial Successive Accepts and Rejects Algorithm for top-k combinatorial bandits. We begin by presenting an efficient estimation algorithm that estimates the expected rewards for all the arms, and then discuss the main algorithm that uses the estimation algorithm in order to find the best subset of $k$ arms. Finally, we bound the sample complexity and regret achieved by the algorithm.

\subsection{Estimation Algorithm}
The first algorithm we discuss suggests an efficient method to estimate the expected rewards of the arms under full bandit feedback. The algorithm gets as inputs a set $\mathcal{N}$ of $n$ arms, a subset size $k$, a level of accuracy $\epsilon$ and a level of confidence $\delta$.
%, and two more dummy variables that will be used later on: a set of accepted arms $\mathcal{A}$ and a set of the top $2k$ arms $\mathcal{T}$. 
The algorithm first partitions $\mathcal{N}$ into sets of size $2k$. In each of those sets, it makes use of the Hadamard matrix as an instructor for the subsets to sample. Let $H$ be the Hadamard matrix of size $2k$, then for each row $H_i\,(i\ne 1)$ the algorithm partitions the arms according to the positive and negative entries in $H_i$. Since in every row exactly half of the entries are positive, the partition forms two sets of size $k$. For $i=1$, $H_1$ has only positive entries, so the algorithm partitions arbitrarily to two sets. Each of these sets is sampled enough times to get a good estimate on its expected reward. Then, the sets' estimated rewards are summed according to their sign in $H$. This way we get a vector $\hat{Z}$ that is equal in expectation to $H\theta$. Finally, to estimate the individual arms' rewards, the algorithm uses the Hadamard matrix inverse $H^{-1}\hat{Z}=\frac{1}{2k}H^\intercal\hat{Z}$, which is the least squares estimator for $\theta$ given $\hat{Z}$.

\begin{algorithm}
\caption{EST1$(\mathcal{N},k,\epsilon,\delta)$}
\label{alg:estimate}
$n=|\mathcal{N}|$;
$m(\epsilon,\delta)=\frac{2}{\epsilon^2}\log\frac{2n}{\delta}$\\
Partition $\mathcal{N}$ into sets of size $2k$: $\mathcal{N}_1\dots \mathcal{N}_\frac{n}{2k}$  \\
\For{$l=1\dots \frac{n}{2k}$}{
Let $\mathcal{N}_l=\{j_1\dots j_{2k}\}$\\
$S_{1,{-1}}=\{j_1,\dots ,j_k\},\,S_{1,{+1}}=\{j_{k+1},\dots ,j_{2k}\}$ \Comment{$i=1$}\\
$S_{i,{b}}=\{j\in\mathcal{N}_l\,|\,H_{ij}=b\}$ \Comment{$i=2\dots 2k,\, b\in\{-1,+1\}$}\\
\For{$i\in [2k],b\in\{-1,+1\}$}{
Sample $S_{i,b}$ for $m=m(\epsilon,\delta)$ times and observe rewards $r_1,\dots,r_m$\\
$\hat\mu_{i,b}=\frac{1}{m}\sum_{t} r_t$
}
$\hat{Z}_1=\hat\mu_{1,{+1}} + \hat\mu_{1,{-1}}$ \Comment{$i=1$}\\
$\hat{Z}_i=\hat\mu_{i,{+1}} - \hat\mu_{i,{-1}}$ \Comment{$i=2\dots 2k,\, b\in\{-1,+1\}$}

$\hat\theta_{\mathcal{N}_l}=\frac{1}{2k}H^\intercal \hat{Z}$
}
\Return $\hat\theta$
\end{algorithm}

\paragraph{Remark 1.} For simplicity, we assume that $2k$ divides $n$. Otherwise, when partitioning the arms in the first step we may repeat arms in the last subset. This increases the number of estimations by at most $2k$, and thus we replace $n$ with $n+2k$ in the number of samples $m(\epsilon,\delta)$. Since $n>2k$, this modification does not change the order of magnitude of the sample complexity and regret.

\paragraph{Remark 2.} We assume that there exists a Hadamard matrix of size $2k$. Otherwise, let $2q\in\mathbb{N}$ be a multiple of $2k$ such that there exists a Hadamard matrix of size $2q$. Partition the arms into subsets of size $2q$ (instead of $2k$), and then in each row the number of positive and negative entries is a multiple of $k$. Then, partition them to $\frac{q}{k}$ sets of size $k$, sample each one separately, and then sum them to get $\hat\mu_{+1}$ and $\hat\mu_{-1}$. This modification changes the sample complexity and regret by at most a constant factor.

\begin{lemma}\label{lemma:PACest1}
For any $\epsilon,\delta>0$ and $k$, and any set of $n$ arms $\mathcal{N}$, EST1 returns an estimated reward vector $\hat\theta$ such that
\[Pr\Big[\forall i,\,|\hat\theta_i-\theta_i|\leq \epsilon\Big]\geq 1-\delta\]
\end{lemma}

\begin{proof}
We first prove that $\hat\theta$ is an unbiased estimator of the reward vector $\theta$. For simplicity fix $N_1=\{1,\dots ,2k\}$ and write $\hat\theta$ instead of $\hat\theta_{N_1}$. 
Note that for each subset $S$, the average $\hat\mu_S=\frac{1}{m}\sum_t r_t$ is an unbiased estimator for the set's reward, namely $\mathbb{E}[\hat\mu_S]=\mu_S=\sum_{i\in S} \theta_i$. As a consequence for each $i\ne 1$, $\hat{Z}_i$ satisfies
\[\mathbb{E}[\hat{Z}_i]= \mu_{i,+1} - \mu_{i,-1} = \sum_{j\in S_{i,+1}} \theta_j - \sum_{j\in S_{i,-1}} \theta_j =\sum_{j=1}^{2k} H_{ij}\theta_j = H_i^\intercal\theta \]
and the same for $i=1$.
Thus $\hat{Z}$ satisfies $\mathbb{E}[\hat{Z}]=H\theta$, and 
$\mathbb{E}[\hat\theta]=\frac{1}{2k}H^\intercal\mathbb{E}[\hat{Z}]=\frac{1}{2k}H^\intercal H\theta=\theta$.

Fix some subset $S$ sampled by the algorithm, and we prove that the estimation noise $\hat\eta_S=\hat\mu_S-\mu_S$ is $\frac{k}{m}$-subgaussian. By definition,
\[\hat\mu_S=\frac{1}{m}\sum_{t=1}^m r_t= \frac{1}{m}\sum_{t=1}^m \sum_{i\in S} X_i= \frac{1}{m}\sum_{t=1}^m \sum_{i\in S} (\theta_i + \eta_{it})=\mu_S+\frac{1}{m}\sum_{t=1}^m \sum_{i\in S} \eta_{it}\]
Since the noise terms $\eta_{it}$ are 1-subgaussians, and we sum over $k$ such terms in each $t$, the total estimation noise is $\frac{k}{m}$-subgaussian. Accordingly, the estimation noise of each $\hat{Z}_i$, given by $\eta_{Z_i}=\hat{Z}_i-\mathbb{E}[\hat{Z}_i]=\eta_{i,+1}-\eta_{i,-1}$ is $\frac{2k}{m}$-subgaussian. Finally, the estimation noise
$\hat\theta_i-\theta_i=\frac{1}{2k}\sum_{j=1}^{2k} H_{ij}\eta_{Z_j}$
is also subgaussian with parameter $\frac{2k}{2km}=\frac{1}{m}$. Thus by Hoeffding inequality for subgaussian random variables,
\[Pr\Big[|\hat\theta_i-\mathbb{E}[\hat\theta_i]|\geq \epsilon\Big]\leq 2\exp\bigg({-\frac{\epsilon^2}{2} m}\bigg)=\frac{\delta}{n}\]
where we substituted the number of samples $m(\epsilon,\delta)$. Finally, the probability of error in one parameter is at most $\frac{\delta}{n}$, and thus by the union bound the probability of error in one parameter or more is at most $\delta$.
\end{proof}

\subsection{Main Algorithm}
We now show how to use the estimation method described above to find the best subset. The algorithm, which we call \emph{Combinatorial Successive Accepts and Rejects} (CSAR), is based on \cite{bubeck2013multiple} for multiple arms identification. CSAR works in phases. In each phase $t$ it maintains a decaying level of accuracy $\epsilon_t$ and confidence $\delta_t$ and uses EST1 to estimate the arms to a given level of accuracy and confidence. Then, it sorts the arms according to their estimations $\hat\theta^t_1\geq \hat\theta^t_2\geq\dots \geq\hat\theta^t_n$, and accepts arms whose estimated reward is bigger than $\hat\theta^t_{k+1}$ by at least $2\epsilon_t$, i.e., $\hat\theta^t_{i}- \hat\theta^t_{k+1}\geq 2\epsilon_t$, as they are optimal with high confidence. Similarly, it rejects arms whose estimated reward is smaller than $\hat\theta^t_{k}$ by at least $2\epsilon_t$. The algorithm proceeds until $n-k$ arms are rejected.

\begin{algorithm}
\caption{Combinatorial Successive Accepts and Rejects (CSAR)}
\label{alg:csar}
$\mathcal{N}^1=\mathcal{N}; \mathcal{A}^1=\emptyset$;
$\epsilon_1 = \frac{1}{2}; \delta_1 = \frac{6}{\pi^2}\delta$\\
\While{$|\mathcal{N}^t\cup\mathcal{A}^t|>k$}{
$\hat\theta^t= EST1(\mathcal{N}^t,k,\epsilon_t,\delta_t)$\\
Sort $\mathcal{N}^t\cup\mathcal{A}^t$ according to $\hat\theta^t$ such that $\hat\theta^t_1\geq \hat\theta^t_2\geq\dots \geq\hat\theta^t_n$ \\
%$\mathcal{T}^t=\{1\dots 2k\}$\\
$\mathcal{A}=\{i\in \mathcal{N}^t\,|\,\hat\theta_i^t - \hat\theta^t_{k+1} > 2\epsilon_t\}$\\
$\mathcal{R}=\{i\in \mathcal{N}^t\,|\,\hat\theta^t_k - \hat\theta_i^t > 2\epsilon_t\}$\\
$\mathcal{A}^{t+1}=\mathcal{A}^t\cup\mathcal{A}$\\
$\mathcal{N}^{t+1}=\mathcal{N}^t\setminus(\mathcal{A}\cup\mathcal{R}) $ \\
$\epsilon_{t+1}=\frac{\epsilon_t}{2}; \delta_{t+1}=\frac{\delta_1}{t^2};t=t+1$
}
\Return{$\mathcal{A}^t\cup\mathcal{N}^t$}
\end{algorithm}

\begin{lemma}\label{thm:pac}
For any $\delta>0$, CSAR with EST1 is $(0,\delta)$-PAC, i.e., it finds the optimal subset with probability at least $1-\delta$.
\end{lemma}

\paragraph{Remark 3.} One can easily modify CSAR to be $(\epsilon,\delta)$-PAC. For that, we provide the algorithm also with a level of accuracy $\epsilon$, and instead of stopping only when $k$ arms are left, we may stop earlier when $\epsilon_t\leq \frac{\epsilon}{2k}$ and return the top $k$ arms according to the last estimation. It is not hard to show that the surviving arms are $2\epsilon_t$ close to the optimal arms and therefore the output is at most $k\epsilon_t=\epsilon$ far from the optimal subset. 
%In this case, the sample complexity is similar to (\ref{eq:sampleComp}), but replacing $\Delta_i$ with $\max\{\Delta_i,\frac{\epsilon}{2k}\}$ since the algorithm might stop earlier.

\subsection{Sample Complexity}
In this section, we bound CSAR's sample complexity in the following theorem.
\begin{theorem}\label{thm:SampleComp}
For any $\delta>0$, the total number of samples performed by CSAR with EST1 is at most
\begin{equation}\label{eq:sampleComp}
    M= O\bigg(\sum_{i=1}^n\Big(\frac{1}{\Delta_i^2}\log\frac{n}{\delta}+\log\log\frac{1}{\Delta_i}\Big)\bigg)
\end{equation}
\end{theorem}

Note that CSAR's sample complexity is comparable with the $O\big(\sum_{i=1}^n \frac{1}{\Delta_i^2}\log\frac{n}{\delta}\big)$ sample complexity of the original Successive Elimination algorithm for best arm identification \citep{even2006action}, and also with algorithms for multiple arms identification \citep{kalyanakrishnan2010efficient, kalyanakrishnan2012pac}, although in these models the agent samples one arm in each round and not $k$ like in the combinatorial model.

To understand how this upper bound scales, consider the following rewards distribution $X_i\sim Ber(\frac{1}{2}+\frac{\epsilon}{k})$ for $i\in[k]$ and $X_i\sim Ber(\frac{1}{2})$ otherwise.
% \[p_i=\begin{cases}
%     \frac{1}{2}+\frac{\epsilon}{k} & i=1\dots k\\
%     \frac{1}{2} & i=k+1\dots n\\
% \end{cases}\]
In this case, for all arms $\Delta_i=\frac{\epsilon}{k}$ and thus the number of samples is bounded by $M=O\big(\frac{nk^2}{\epsilon^2}\log\frac{n}{\delta}\big)$ (ignoring $\log \log$ terms).

To bound the sample complexity, we first prove the following lemma that bounds the cumulative number of times $M_i$ each arm is sampled until it is accepted or rejected. The theorem follows immediately by summing $M_i$ over all arms and dividing by $k$ since each subset sampled by the algorithm consists of $k$ arms.

\begin{lemma}\label{lemma:armSampleComp}
For each arm $i\in[n]$, the number of times it is sampled until it is rejected (if it is sub-optimal) or accepted (if it is optimal) is bounded by
\[M_i \leq \frac{Ck}{\Delta_i^2}\big(\log\frac{2n}{\delta}+2\log\log\frac{1}{\Delta_i}\big)\]
\end{lemma}

\begin{proof}
Let $i$ be an arm, and let $T_i$ be the phase it is accepted or rejected. In every phase $t<T_i$, arm $i$ is sampled as part of $2k$ subsets and each subset is sampled $m(\epsilon_t,\delta_t)$ times where $\epsilon_t=2^{-t}$ and $\delta_t=\frac{\delta_1}{t^2}$, thus we have
\begin{equation}\label{eq:armSampleComp}
\begin{split}
M_i &= \sum_{t=1}^{T_i}2k m(\epsilon_t,\delta_t)=\sum_{t=1}^{T_i}\frac{4k}{\epsilon_t^2}\log\frac{2n}{\delta_t}= 4k\sum_{t=1}^{T_i}(2^{t})^2\log\frac{t^2 2n}{\delta_1}=\\
&=4k\Big(\sum_{t=1}^{T_i}4^{t}\log t^2+\log\frac{2n}{\delta_1}\sum_{t=1}^{T_i}4^{t}\Big)\leq
Ck\big(2\log T_i+\log\frac{2n}{\delta}\big)\cdot 4^{T_i}
\end{split}
\end{equation}

We now bound the phase $T_i$ when $i$ is rejected. We discuss the case that $i$ is sub-optimal, but the analysis for optimal arms is similar.
Assuming all arms are estimated accurately (see Appendix \ref{proof:CsarPac}), then for any phase $t$ and any arm $i$ we have $|\hat\theta_i^t-\theta_i|\leq \epsilon_t$. That also implies that the difference between the real $k$th arm and the arm estimated to be in the $k$th place satisfies $|\hat\theta_k^t-\theta_k|\leq \epsilon_t$, since mixing the order of the arms can happen only between arms that are within the same $\epsilon_t$-neighborhood.
As long as $i$ was not rejected, i.e., $\forall t=1\dots T_i-1$, it holds that
\[2\epsilon_t\geq \hat\theta_k^t-\hat\theta_i^t\geq (\theta_k-\epsilon_t)-(\theta_i+\epsilon_t)=(\theta_k-\theta_i)-2\epsilon_t=\Delta_i-2\epsilon_t \]
Substituting $\epsilon_t=2^{-t}$ we get $\Delta_i\leq 4\epsilon_t=4\cdot 2^{-t}$.
This is true also for $t=T_i-1$, and thus we get $T_i\leq\log\frac{4}{\Delta_i}$. Substituting $T_i$ in (\ref{eq:armSampleComp}) yields the desired bound.
\end{proof}

\subsection{Regret}
We now analyze the regret.
Notice that while CSAR aims to minimize the sample complexity, it does not minimize the regret. This is because at each time the algorithm chooses a subset, the regret it achieves is affected not only by the arms it selected, but also by the arms it did not select. In other words, the gap that should be considered is between the sub-optimal arm $i\in\{k+1,\dots ,n\}$ that was actually selected and the optimal arm $j\in\{1, \dots, k\}$ that would have been selected instead. We denote this gap by $\Delta_{j:i}=\theta_j-\theta_i$. Using this notation, we may bound the regret of the algorithm. 

\begin{theorem}\label{thm:weakRegret}
For any $n,k\leq\frac{n}{2}$ and $T$ the regret of CSAR with EST1 is at most
\[R=O\Bigg(\sum_{i=k+1}^n  \frac{\Delta_{1:i}}{\Delta_{k:i}^2} k\log T\Bigg)\]
\end{theorem}

Note that this bound is tight for CSAR with EST1. Consider the problem instance where each arm $i\in [n]$ is associated with a normal random variable $X_i\sim\mathcal{N}(\theta_i,1)$, where
\begin{equation}\label{instance:regret}
    \theta_i=\begin{cases}
    \Delta_{+} & i<k\\
    0 & i=k\\
    -\Delta_{-} & i>k
    \end{cases}
\end{equation}
and assume $\Delta_+ \gg \Delta_-$. On this problem instance, CSAR will accept the first $k-1$ arms after a small number of iterations. Then, for the rest of the run it will sample only arms with expected reward of at most $0$. In each call to EST1 each of the $n-k+1$ arms is sampled $\frac{8k}{\epsilon_t^2}\log\frac{n}{\delta_t}$ times, and it keeps being sampled until $2\epsilon_t<\Delta_{-}$. Therefore, the total regret of the algorithm is $\Theta \Big(\frac{\Delta_{+}}{\Delta_{-}^2}(n-k)k\log\frac{n}{\delta}\Big)$. In the following section we discuss a modification for the algorithm that helps achieve smaller regret.

\subsection{Modified algorithm with improved regret}
The reason for the $\Delta_{1:i}$ factors in Theorem \ref{thm:weakRegret} is due to the fact that when we identify an arm as optimal, we stop sampling it, and thus suffer regret for its absence. Instead, we consider the following modification for the algorithm in order to improve the regret. When we accept an arm, instead of preventing it from being sampled, we fix it. Namely, we sample it in every subset until the end of the run. This will assure that we suffer gaps such as $\Delta_{1:i}$ only for a small number of rounds. 

Accordingly, we modify the estimation algorithm to support fixed arms. Now, the algorithm gets as input also a set $\mathcal{A}$ of accepted arms that must be sampled in each subset. Instead of using the Hadamard matrix of size $2k$, it takes a smaller one of size $2k'$ where $k'=k-|\mathcal{A}|$ is the number of arms that can be sampled in each subset after keeping room for the fixed arms. Most of the algorithm remains the same, except for the need to have good estimations for the fixed arms' expected rewards. This is because it needs to eliminate those rewards from the sampled subsets and stay only with the arms that should be estimated. For that, we provide it with a set ${\mathcal{T}}$ of the top $2k$ arms, according to the last phase estimations, and run EST1 separately on them.

\begin{algorithm}
\caption{EST2$(\mathcal{N},k,\epsilon,\delta,\mathcal{A},\mathcal{T})$}
\label{alg:estimate-fixed}
$n=|\mathcal{N}|;\,k'=k-|\mathcal{A}|;\,m(\epsilon,\delta,k')=\frac{2k}{k'}\frac{2}{\epsilon^2}\log\frac{2n}{\delta}$\\
$\hat\theta_1\dots\hat\theta_{2k} = EST1({\mathcal{T}},\epsilon,\delta,\mathcal{A},\mathcal{T})$\\
Partition $\mathcal{N}$ into sets of size $2k'$: $\mathcal{N}_1\dots \mathcal{N}_\frac{n}{2k'}$  \\
\For{$l=1\dots \frac{n}{2k'}$}{
Let $\mathcal{N}_l=\{j_1\dots j_{2k'}\}$\\
$S'_{1,{-1}}=\{j_1,\dots ,j_{k'}\},\,S'_{1,{+1}}=\{j_{k'+1},\dots ,j_{2k'}\}$\Comment{$i=1$}\\
$S'_{i,{b}}=\{j\in\mathcal{N}_l\,|\,H_{ij}=b\}$ \Comment{$i=2\dots 2k',\, b\in\{-1,+1\}$}\\
\For{$i\in [2k'],b\in\{-1,+1\}$}{
$S_{i,b}=S_{i,b}'\cup\mathcal{A}$\\
Sample $S_{i,b}$ for $m=m(\epsilon,\delta,k')$ times and observe rewards $r_1,\dots,r_m$\\
$\hat\mu_{i,b}=\frac{1}{m}\sum_{t} r_t$
}
$\hat{Z}_1=\hat\mu_{1,+1} + \hat\mu_{1,-1} - 2\sum_{a\in\mathcal{A}} \hat\theta_a$ \Comment{$i=1$}\\
$\hat{Z}_i=\hat\mu_{i,{+1}} - \hat\mu_{i,{-1}}$ \Comment{$i=2\dots 2k',\, b\in\{-1,+1\}$}

$\hat\theta_{\mathcal{N}_l}=\frac{1}{2k'}H^\intercal \hat{Z}$
}
\Return $\hat\theta$
\end{algorithm}

%We show that using EST2 instead of EST1 improves CSAR's regret.
\begin{theorem}\label{thm:regretUpperBound}
For any $n,k$ and time horizon $T$, the regret of CSAR with EST2 is at most
\begin{equation}\label{eqn:regret}
    R= O\Bigg(\bigg(\sum_{i=1}^k\frac{\Delta_{i:(k+i)}}{\Delta^2} + \sum_{i=k+1}^n  \frac{1}{\Delta_i} \bigg) k\log T \Bigg)
\end{equation}
\end{theorem}

Note that this is an improvement over CSAR with EST1. For example, on problem (\ref{instance:regret}), the first arms will be fixed after a small number of rounds and the regret will be
\[R=\Theta\bigg(\Big(k\frac{\Delta_+}{\Delta_-^2}+\frac{n-k}{\Delta_-}\Big)k\log T\bigg)\]
which is better than $\Theta \Big(\frac{\Delta_{+}}{\Delta_{-}^2}(n-k)k\log T\Big)$ as long as $\frac{\Delta_+}{\Delta_-}\gtrapprox 1+\frac{k}{n}$.

To prove the upper bound on the regret of CSAR with EST2, we first prove the following lemma that bounds the regret caused by each sub-optimal arm.

\begin{lemma}\label{lemma:individualArmRegret}
The regret of any sub-optimal arm $i$ is at most
$R_i= O\bigg(\frac{1}{\Delta_i}k\log\frac{n}{\delta}\bigg)$
\end{lemma}

\begin{proof}
Since all expressions in this proof depend on a factor of $Ck\log\frac{2n}{\delta}$, we omit it along the proof and multiply by it at the end.
Fix a sub-optimal arm $i$. By Lemma \ref{lemma:armSampleComp}, the number of times $i$ is chosen until it is rejected is at most $M_i\leq \frac{1}{\Delta_i^2}$.
We split the optimal arms $\{1,\dots , k\}$ according to $\Delta_i$, and bound separately the regret $R_i^<$ caused by missing an optimal arm $j\in[k]$ with $\Delta_j\leq\Delta_i$, and the regret $R_i^>$ for the arms $j\in[k]$ with $\Delta_j>\Delta_i$.
\begin{itemize}
    \item For any $j\in[k]$ such that $\Delta_j\leq \Delta_i$, the maximal gap we pay for taking arm $i$ instead of arm $j$ is at most $\Delta_{j:i}\leq 2\Delta_i$, and thus the regret of such case is bounded by \[R_i^<\leq M_i\Delta_{j:i}\leq \frac{1}{\Delta_i^2}\cdot (\Delta_j+\Delta_i)\leq \frac{1}{\Delta_i^2}\cdot 2\Delta_i= \frac{2}{\Delta_i}\]
    \item For any $j\in[k]$ such that $\Delta_j>\Delta_i$, arm $j$ is accepted at some point before arm $i$ is rejected, thus at some point we can be sure that arm $i$ is not played instead of arm $j$. Let $l=\arg\min_{j:\Delta_j>\Delta_i}\Delta_j$. We showed that each optimal arm $j$ is accepted at phase $T_j\leq\log\frac{4}{\Delta_i}$, thus we can write the regret of arm $i$ up to phase $T_j$ as
    \begin{equation*}
    \begin{split}
        R_i^>\leq &M_1\Delta_{1:i}+ (M_2-M_1)\Delta_{2:i} + \dots + (M_l-M_{l-1})\Delta_{l:i}\\
        = & M_1(\Delta_{1:i}-\Delta_{2:i}) + M_2(\Delta_{2:i}-\Delta_{3:i}) + \dots + M_j\Delta_{j:i} \\ = &\sum_{j=1}^{l-1} M_j (\Delta_{j:i}-\Delta_{(j+1):i})+ M_l\Delta_{l:i}\\
    % \end{split}
    % \end{equation*}
    % Note $\Delta_{j:i}=\Delta_{j}+\Delta_{i}-\Delta$ and substitute $M_i(T_j)\leq\frac{1}{\Delta_j^2}$ to get
    % \begin{equation*}
        % \begin{split}
        \leq&\sum_{j=1}^{l-1} \frac{\Delta_{j}-\Delta_{j+1}}{\Delta_j^2} + \frac{\Delta_{l}+\Delta_i}{\Delta_l^2} \\
        \leq & \int_{\Delta_l}^{\Delta_1} \frac{1}{x^2}dx + \frac{2\Delta_{l}}{\Delta_l^2}  =\bigg(\frac{1}{\Delta_l}-\frac{1}{\Delta_1}\bigg) +\frac{2}{\Delta_l}\leq \frac{3}{\Delta_l}\leq \frac{3}{\Delta_i}
        \end{split}
    \end{equation*}
\end{itemize}
To sum up, arm $i$'s contribution to the regret is $R_i=R_i^<+R_i^>$ multiplied by $Ck\log\frac{2n}{\delta}$.
\end{proof}

Theorem \ref{thm:regretUpperBound} is implied by Lemma \ref{lemma:individualArmRegret} by summing $R_i$ over all sub-optimal arms, in addition to the regret accumulated by estimating the top $2k$ arms in each phase until the end of the run. Each of them is sampled $\frac{ck}{\Delta^2}\log\frac{2n}{\delta}$ times for some $c>0$. As they are the top $2k$ arms with high probability, the worst subset that can be sampled is $\{k+1,\dots,2k\}$, and the gap between it and the optimal subset is $\sum_{i=1}^k \Delta_{i:(k+i)}$. Thus their regret is at most $\Big(\sum_{i=1}^k \Delta_{i:(k+i)}\Big) \frac{ck}{\Delta^2}\log\frac{2n}{\delta}$, and together with $\sum_{i=k+1}^n R_i$ we get Theorem \ref{thm:regretUpperBound}.

% \begin{proof}
% The total regret of the algorithm is compounded of two parts. The first is the regret accumulated in the elimination process of the sub-optimal arms. Lemma \ref{lemma:individualArmRegret} provides a bound on the contribution of every sub-optimal arm to the regret, thus by summing over all sub-optimal arms we have

% \[R_{sub}=\sum_{i=k+1}^n R_i\leq O\bigg(\sum_{i=k+1}^n \frac{1}{\Delta_i}k\log\frac{n}{\delta}\bigg)\]

% The second part of the regret is the one accumulated in the process of estimating the top $2k$ arms. Since they are being estimated in each phase until the end of the run, that is until the arm with the minimal $\Delta$ is recognized, the number of times each one of them is sampled is proportional to $\frac{1}{\Delta^2}$. As they are the top $2k$ arms with high probability, the worst subset with regard to regret that can be sampled is the subset $\{k+1,\dots,2k\}$. The gap between this subset and the optimal subset is $\sum_{i=1}^k \Delta_{i:(k+i)}$ and thus 
% \[R_{top}\leq\sum_{i=1}^k \frac{\Delta_{i:(k+i)}}{\Delta^2}Ck\log\frac{2n}{\delta}\]

% Finally, the regret is given by $R=R_{sub}+R_{top}$.
% \end{proof}

\begin{corollary}\label{cor:IndRegret}
The distribution-independent regret is at most $O\big(k\sqrt{nT\log T}\big)$
\end{corollary}

CSAR's distribution-independent regret is bigger by factor $\sqrt{k}$ than the $\Omega\big(\sqrt{knT}\big)$ lower bound for semi-bandits in \cite{lattimore2018toprank} (ignoring $\log$ terms). In many cases it is reasonable to assume $k=O(1)$ which makes the bounds tight. As for the dependence on $k$, we leave the search for tighter bounds for further research. 

Assuming all gaps are equal to $\Delta$, the regret in (\ref{eqn:regret}) can be written as
$R= O\bigg(\frac{nk}{\Delta} \log T\bigg)$.
%To stress the necessity of the factor $k$ in the regret upper bound, 
We prove that it is tight for CSAR.

\begin{lemma}
For any $n,k$ and time horizon $T$, there exists a distribution over the assignment of rewards such that the regret of CSAR with EST2 is at least
\[R=\Omega\bigg(\frac{nk}{\Delta}\log T\bigg)\]
\end{lemma}

\begin{proof}
Consider the following example. Each arm $i\in [n]$ is associated with a Gaussian random variable $X_i$ where $X_i\sim\mathcal{N}(\Delta,1)$ if $i\leq k$ and $X_i\sim\mathcal{N}(0,1)$ otherwise.
% \[\theta_i=\begin{cases}
% \Delta & i\leq k\\
% 0 & i>k
% \end{cases}\]
%namely, for the optimal arms the expected reward is $\Delta$ and for the rest it is $0$. 
Similarly to Lemma $\ref{lemma:armSampleComp}$, the best arms will be identified only when $\Delta>2\epsilon_t$ which implies that the number of phases is $T>\Omega(\log\frac{1}{\Delta})$, and since no arm is accepted or rejected until this phase the total number of samples is $\Omega\big(\frac{n}{\Delta^2}\log\frac{n}{\delta}\big)$. Additionally, each subset has a gap of up to $k\Delta$. Thus, the total regret is 
$R=\Omega\big(k\Delta\cdot\frac{n}{\Delta^2}\log\frac{n}{\delta}\big)$
which proves that the regret upper bound is tight.
\end{proof}

The following table summarizes the theoretical bounds of CSAR in comparison to other algorithms for top-k combinatorial bandits with full-bandit feedback.

\begin{center}
 \begin{tabular}{||c | c c c||} 
 \hline
 Algorithm & Sample Complexity & Depend. Regret & Ind. Regret \\ [1ex] 
 \hline\hline
 \textbf{CSAR} & $O\big(\frac{n}{\Delta^2}\log\frac{n}{\delta}\big)$ & $O\big(\frac{nk}{\Delta}\log T\big)$ & $O\big(k\sqrt{nT}\big)$ \\ [1ex] 
 \hline
 \cite{agarwal2018regret} & -- & -- & $O(k^\frac{1}{2}n^\frac{1}{3}T^\frac{2}{3})$ \\ [1ex] 
 \hline
 \cite{dani2008stochastic} & -- & $O\big(\frac{n^3}{\Delta}\log^3 T\big)$ & $O\big(n\sqrt{nT}\big)$ \\ [1ex] 
 \hline
 \cite{kuroki2019polynomial} & $O\big(\frac{n^\frac{5}{4}k^4}{\Delta^2}\log\frac{n}{\delta}\big)$ \footnote{See Appendix \ref{appendix:kuroki} for elaboration on this bound.}  & -- & -- \\ [1ex] 
 \hline
\end{tabular}
\end{center}
\section{Lower Bound}
In this section we bound the minimal number of samples necessary to identify the best subset under full-bandit feedback.

One might wonder if the $\Omega\big(\frac{n}{\epsilon^2}\log\frac{1}{\delta}\big)$ lower bound for best arm identification \citep{mannor2004sample} or $\Omega\big(\frac{n}{\epsilon^2}\log\frac{k}{\delta}\big)$ for multiple arms identification \citep{kaufmann2013information, kalyanakrishnan2012pac, chen2014combinatorial} applies also for combinatorial bandits. The answer is not immediate. Intuitively, sampling $k$ arms together might provide more information, so that hypothetically less samples can be used to find the best subset. For example, if the goal is to detect an unknown number of counterfeit coins out of $n$ coins, and the agent is allowed to weigh any number of coins, then there exists an algorithm that identifies the counterfeit coins using only $\Theta\big(\frac{n}{\log n}\big)$ weighings, with or without the presence of noise \citep{erdos1963two,soderberg1963combinatory,bshouty2012coin}.
%Formally, since each sample is the sum of $k$ variables, it contains $\log k$ bits of information (up to the number of bits needed for representing each arm). Thus, if the lower bound for Explore-k is given by $\Omega(\frac{n}{\epsilon^2})$, a naive approach could expect to solve the combinatorial problem within $\Omega(\frac{n}{\epsilon^2\log k})$ samples.

% We demonstrate this intuition with an example. There are $n$ arms, each arm is associated with a constant number (instead of random variable), and the agent's objective is to find the best $k$. For simplicity, assume that there are $k$ arms with value $\epsilon$ and the rest are $0$s. The following algorithm finds the top $k$ using $O(\frac{n}{k}+k\log k)$ queries. First partition the arms into $\frac{n}{k}$ sets, query for each set and eliminate all the ones with outcome $0$. We are left with at most $k$ sets. In each of them, use binary search to find the hiding $\epsilon$s. This algorithm uses $\frac{n}{k}$ queries for the first step and $\log k$ queries for binary search in each of the $k$ sets. Thus the total number of queries is $O(\frac{n}{k}+k\log k)$ which is generally smaller than $\Omega(\frac{n}{\epsilon^2})$. However, this example assumed constant arms. It raises the question what is the case when the arms are random variables.

Despite the discussion above, the following theorem proves a lower bound of $\Omega\big(\frac{n}{\epsilon^2}\big)$ samples for combinatorial bandits, similar to the bounds for best- and multiple-arms identification tasks. 
%This bound shows that CSAR is optimal with respect to sample complexity (up to $\log$ facotrs). 

\begin{theorem}\label{thm:lowerBound}
For any $n$ and $k\leq\frac{n}{2}$, and for any $0<\epsilon, \delta<\frac{1}{2}$, there exists a distribution over the assignment of rewards such that the sample complexity of any $(\epsilon,\delta)$-PAC algorithm is at least
\[M=\Omega\Big(\frac{n}{\epsilon^2}\Big)\]
\end{theorem}

The proof is based on \citet[chap.~2]{slivkins2019introduction}, but generalized for the combinatorial setting.  It defines two problem instances with small KL-divergence between them and shows that any algorithm that uses less samples than required is wrong with high probability.

\section{Experiments}

We compared our algorithm to other methods also experimentally, on simulated data. We conducted two experiments, one for the sample complexity and one for the regret. We describe here the experiments briefly, for more details see Appendix \ref{sec:experiments}.

For the sample complexity, we evaluate the accuracy of different sampling methods in comparison to EST1.
Figure \ref{fig:exp}(a) shows the mean square error of EST1 with Hadamard matrices along with two other sampling methods. It can be seen that Hadamard significantly outperforms the others. For the regret, we compared CSAR with the Sort \& Merge algorithm in \cite{agarwal2018regret}. Figure \ref{fig:exp}(b) shows the cumulative regret as a function of time for both algorithms. It can be seen that CSAR achieves significantly lower regret than Sort \& Merge.

\begin{figure}
    \centering
    {\includegraphics[width=6.5cm]{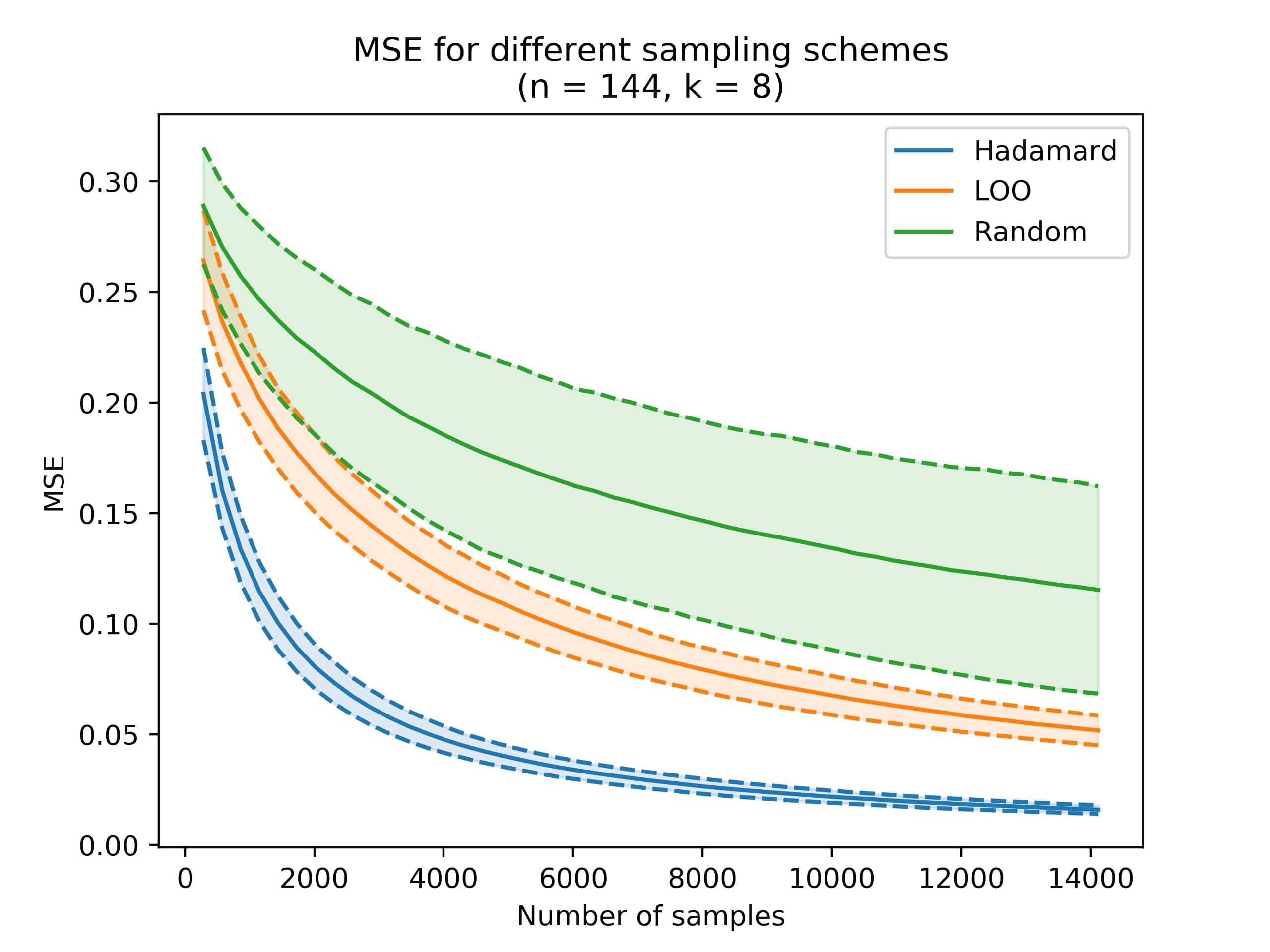} }%
    \qquad
    {\includegraphics[width=6.5cm]{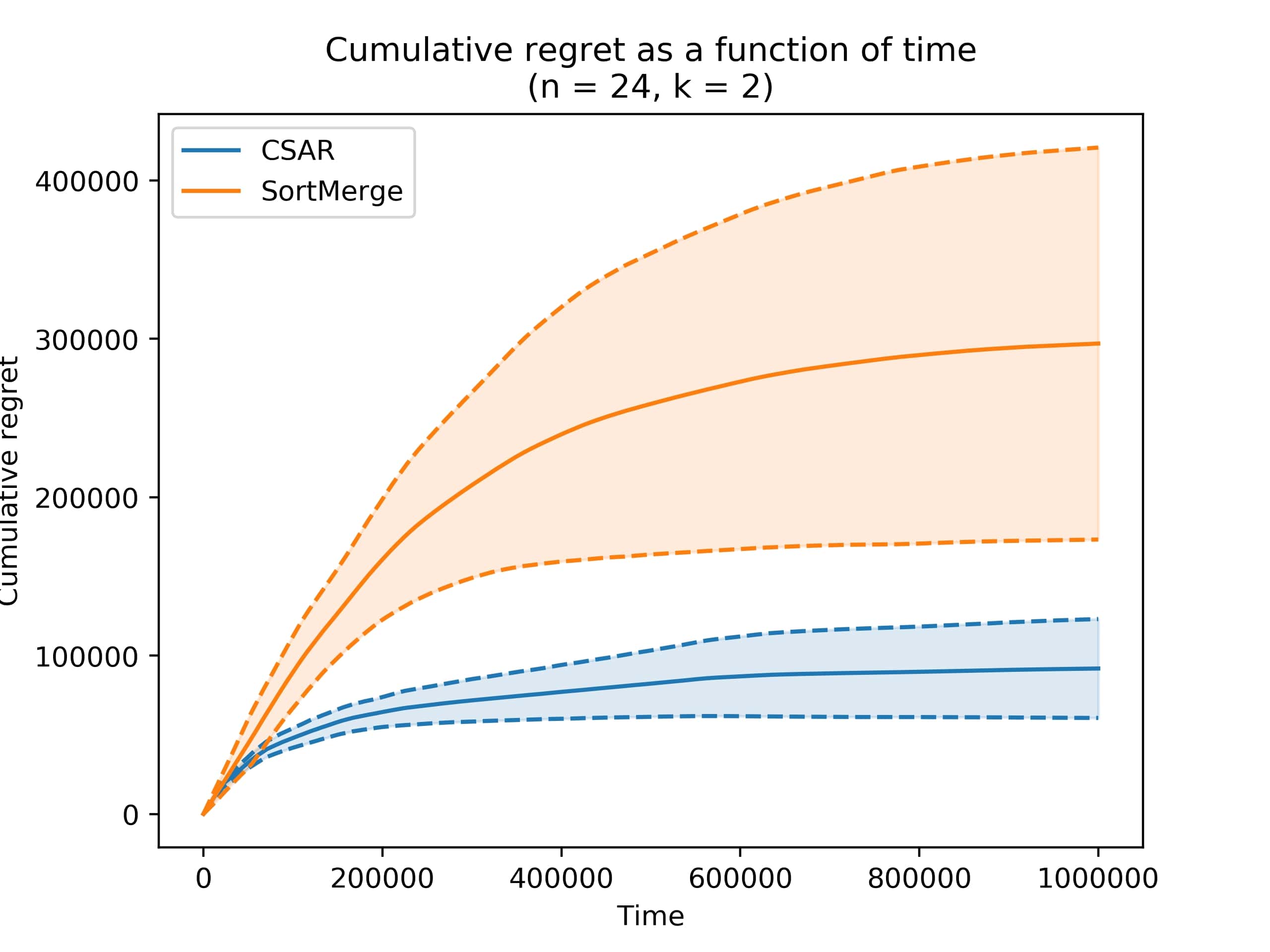} }%
    \caption{(a) Comparing sampling schemes (Hadamard, LOO and random).
    (b) Comparing regret as a function of time horizon (CSAR vs Sort \& Merge).}%
    \label{fig:exp}%
\end{figure}

\section{Discussion and Conclusions}
In this work we proposed a novel algorithm for top-k combinatorial bandits with full-bandit feedback. We presented the Combinatorial Successive Accepts and Rejects (CSAR) algorithm, and showed that it is $(0,\delta)$-PAC with sample complexity $O\big( \frac{n}{\Delta^2} \log\frac{n}{\delta}\big)$ and regret $O\big(\frac{nk}{\Delta} \log T\big)$ for time horizon $T$. For the sample complexity, we also proved a lower bound of $\Omega\big(\frac{n}{\epsilon^2}\big)$. To the best of our knowledge, this is the first lower bound for sample complexity of combinatorial bandits with full-bandit feedback. 
%Furthermore, we showed that the regret is bounded by $O\big(k\sqrt{nT}\big)$ for time horizon $T$, which is bigger by factor $\sqrt{k}$ than the lower bound $\Omega\big(\sqrt{knT}\big)$ in \cite{lattimore2018toprank}. We leave the search for tighter bounds for further research. 

In addition, we tested our results empirically. First, we tried three sampling methods and showed that our novel method using Hadamard matrices achieves bigger accuracy within less samples, comparing to the baselines. Second, we compared the cumulative regret to \cite{agarwal2018regret}, and illustrated that CSAR outperforms the latter.

\acks{This work was supported part by the Yandex Initiative in Machine Learning.}


\newpage
\vskip 0.2in
\bibliography{references}

\appendix

\section{Proofs}
\subsection{Proof Lemma \ref{thm:pac} (PAC)}\label{proof:CsarPac}
\begin{proof}
For each phase $t$, define the event $E_t$ that at least one arm is estimated poorly, i.e.
$E_{t}=\{\exists i\,|\hat\theta_{t}^i-\theta_{t}^i|>\epsilon_t\}$,
and let $E=\bigcup_t E_t$. Note that CSAR is wrong only if at some phase it rejects an optimal arm or accepts a sub-optimal arm. This might happen only under the event $E$. Hence, the probability that CSAR is wrong is bounded by the probability of $E$.
By Lemma 1 for any $t$, $Pr[E_t]\leq\delta_t$ and thus by the union bound
\[Pr[E]\leq\sum_t Pr[E_t]\leq \sum_t \delta_t= \sum_{t=1}^{T}\frac{\delta_1}{t^2}\leq \delta_1 \sum_{t=1}^{\infty} \frac{1}{t^2}=\frac{6}{\pi^2}\delta\cdot\frac{\pi^2}{6}=\delta\]
Accordingly, the algorithm returns the optimal subset with probability at least $1-\delta$.
%For the rest of the proofs we assume that $E$ did not happen.
\end{proof}

\subsection{Proof EST2 Correctness}
\begin{lemma}
For any $\epsilon,\delta>0$ and any set of $n$ arms $\mathcal{N}$ and set of accepted arms $\mathcal{A}$, EST2 returns an estimated reward vector $\hat\theta$ such that
\[Pr\Big[\forall i,\,|\hat\theta_i-\theta_i|\leq \epsilon\Big]\geq 1-\delta\]
\end{lemma}

\begin{proof}
The proof is similar to Lemma \ref{lemma:PACest1}, so we only stress the differences.
First, when proving that $\hat\theta$ is an unbiased estimator of $\theta$, we have $S=S'\cup\mathcal{A}$, hence for each $i\neq 1$,
\[\begin{split}
    \mathbb{E}[\hat{Z}_i]= &\mu_{i,+1} - \mu_{i,-1} = \sum_{j\in S_{i,+1}'} \theta_j + \sum_{j\in \mathcal{A}} \theta_j - \bigg(\sum_{j\in S_{i,-1}'} \theta_j + \sum_{j\in \mathcal{A}} \theta_j \bigg) \\ =&\sum_{j\in S_{i,+1}'} \theta_j - \sum_{j\in S_{i,-1}'} \theta_j=\sum_{j=1}^{2k'} H_{ij}\theta_j = H_i^\intercal\theta
\end{split} \]
and for $i=1$
\[\begin{split}
    \mathbb{E}[\hat{Z}_1]=& \mu_{1,-1} + \mu_{1,+1} - 2\sum_{j\in\mathcal{A}} \theta_j \\ =&\sum_{j\in S_{1,+1}'} \theta_j + \sum_{j\in \mathcal{A}} \theta_j + \sum_{j\in S_{1,-1}'} \theta_j + \sum_{j\in \mathcal{A}} \theta_j - 2\sum_{j\in\mathcal{A}} \theta_j \\ =&\sum_{j\in S_{1,+1}'} \theta_j + \sum_{j\in S_{1,-1}'} \theta_j=\sum_{j=1}^{2k'} H_{1j}\theta_j = H_1^\intercal\theta
\end{split} \]

Second, we prove each $\eta_{\hat{Z}_i}$ is subgaussian. For any $i\neq 1$ the proof remains the same. For $i=1$,
\begin{equation}\label{eqn:Z1}
    \hat{Z}_1= \hat\mu_{1,-1} + \hat\mu_{1,+1} - 2\sum_{j\in\mathcal{A}} \hat\theta_j
\end{equation}
Thus the noise consists of $\eta_{1,-1}, \,\eta_{1,+1}$ which are $\frac{k}{m}$-subgaussians, and the noise of each $\hat\theta_j$. We proved in Lemma \ref{lemma:PACest1} that the latter is $\frac{1}{m}$-subgaussian, and therefore when summing at most $k$ such terms and multiplying by $2$ we get that the last term in (\ref{eqn:Z1}) is $\frac{2k}{m}$-subgaussian. Summing all terms, we get that $\eta_{\hat{Z}_i}$ is $\frac{4k}{m}$-subgaussian.

Finally, note that the estimation noise of $\hat\theta$, given by
$\hat\theta_i-\theta_i=\frac{1}{2k'}\sum_{j=1}^{2k'} H_{ij}\eta_{Z_j} $,
is $\frac{4k}{2k'm}$-subgaussian. Thus for $m=\frac{2k}{k'}\frac{2}{\epsilon^2}\log\frac{2n}{\delta}$ it holds
$Pr\Big[|\hat\theta_i-\mathbb{E}[\hat\theta_i]|\geq \epsilon\Big]\leq\frac{\delta}{n}$, and by the union bound the toal probability of error is at most $\delta$.
\end{proof}

\subsection{Proof Theorem \ref{thm:weakRegret} (CSAR with EST1's Regret)}

\begin{proof}
Since only sub-optimal arms are responsible for regret, we consider only them. For any sub-optimal arm $i$, its maximal gap is $\Delta_{1:i}$. Hence, the total regret is given by
$R\leq\sum_{i=k+1}^n M_i\Delta_{1:i}$
where $M_i$ is the number of times arm $i$ is sampled.
In order to translate this bound to terms of the time horizon $T$, recall that if the algorithm goes wrong, it might suffer a regret of $kT$, and this happens with probability $\delta$. To avoid it, we take $\delta=\frac{1}{kT}$. Using Lemma \ref{lemma:armSampleComp} to bound $M_i$ we get the desired regret bound.
\end{proof}

\subsection{Proof Corollary \ref{cor:IndRegret} (Distribution-Independent Regret)}

\begin{proof}
Consider the $(\epsilon,\delta)$-PAC variant of CSAR that stops exploring when $\epsilon_t\leq\epsilon$ and then keep selecting the best $k$ estimated arms for the rest of the time horizon $T$ (see Remark 3). Note that when the exploration stops, the arms estimations are at most $\epsilon$ far from their real values, according to Lemma \ref{lemma:PACest1}. Hence, the gap between the optimal subset and any subset of surviving arms is at most $k\epsilon$, and thus their regret is at most $R^<\leq k\epsilon T$.
This should be added to the regret caused by the arms that were eliminated up to this stage. According to Lemma \ref{lemma:individualArmRegret}, the contribution of a sub-optimal arm $i$ to the regret is bounded by $\frac{k}{\Delta_i}\log \frac{n}{\delta}$. Since it was eliminated before phase $t$, it must hold that $\Delta_i>\epsilon$. The number of eliminated arms is clearly bounded by $n$, and thus their contribution to the regret is
\[R^>\leq C\sum_{i:\Delta_i>\epsilon}\frac{k}{\Delta_i}\log T\leq C\sum_{i:\Delta_i>\epsilon}\frac{k}{\epsilon}\log T \leq C\frac{nk}{\epsilon}\log T\]
for some constant $C$. Concluding both parts of regret, we get
$R\leq C\frac{nk}{\epsilon}\log T + \epsilon k T$.
This is true as long as the number of rounds is at least the number of samples done by CSAR up to phase $t$, that is $T\geq C'\frac{n}{\epsilon^2}\log \frac{n}{\delta}$ for some constant $C'$.
Thus for $\epsilon=\sqrt{\frac{C'n}{T}\log \frac{n}{\delta}}$ and $\delta=\frac{1}{kT}$ we get the desired regret bound.
\end{proof}
\subsection{Proof Theorem \ref{thm:lowerBound} (Lower Bound)}
We prove the lower bound in a few steps. We first prove a bound of $\Omega\big(\frac{k}{\epsilon^2}\big)$ for $k\leq\frac{n}{2}$, then prove a stronger bound of $\Omega\big(\frac{n}{\epsilon^2}\big)$ for $k\leq\frac{n}{24}$, and finally sum up both proofs to get the desired bound.
\subsubsection*{Lower bound for $k\leq\frac{n}{2}$}
To prove the lower bound for $k\leq\frac{n}{2}$, we first define the following profiles with $n=2k$ arms.
\begin{equation}\label{dfn:probLowerBoundk2}
\begin{split}
    \mathcal{I}_1=\{X_i\sim Ber(p_i^1)\}_{i=1}^{2k}\,\,\,where\,\,\,p_i^1=\begin{cases}
    \frac{1}{2}+\frac{\epsilon}{k} & i=1\dots k \\
    \frac{1}{2}-\frac{\epsilon}{k} & i=k+1\dots 2k
    \end{cases}\\
    \mathcal{I}_2=\{X_i\sim Ber(p_i^2)\}_{i=1}^{2k}\,\,\,where\,\,\,p_i^2=\begin{cases}
    \frac{1}{2}-\frac{\epsilon}{k} & i=1\dots k \\
    \frac{1}{2}+\frac{\epsilon}{k} & i=k+1\dots 2k
    \end{cases}
\end{split}
\end{equation}
In what follows we assume that $i\in\{1,2\}$ is selected randomly and the agent gets to play against profile $\mathcal{I}_i$ without knowing the value of $i$.

\begin{lemma_thm}\label{lemma:simuAlg}
Any algorithm A that runs on problem (\ref{dfn:probLowerBoundk2}) and selects any subset $S\subset[2k]$ of size $k$ can be simulated by an algorithm A' that selects only $\mathcal{K}_1=\{1\dots k\}$ and $\mathcal{K}_2=\{k+1\dots 2k\}$ with the same amount of samples.
\end{lemma_thm}
\begin{proof}
Fix an algorithm A and a subset $S$. Let $(S_1,S_2)$ be a partition of $S$, i.e., $S_1\cap S_2=\emptyset, S_1\cup S_2=S$, such that $S_1\subseteq \mathcal{K}_1$ and $S_2\subseteq \mathcal{K}_2$. Assume without loss of generality $|S_1|=s_1, |S_2|=s_2$ and $s_1\geq s_2$. Then, there are at least $s_2$ arms in $S$ with mean $\frac{1}{2}+\frac{\epsilon}{k}$ and $s_2$ arms with mean $\frac{1}{2}-\frac{\epsilon}{k}$. As we observe only the sum of the rewards, which is in this case $s_2\big(\frac{1}{2}+\frac{\epsilon}{k}\big)+s_2\big(\frac{1}{2}-\frac{\epsilon}{k}\big)=\frac{2s_2}{2}$, we may simulate these $2s_2$ arms with the same amount of fair coins with probability $\frac{1}{2}$.

We now show how to simulate the distribution of the rest $s=s_1-2s_2\leq k$ arms in $S_1\subseteq \mathcal{K}_1$ using one sample of $\mathcal{K}_1$. Sample $\mathcal{K}_1$ once and let $r$ be the outcome. Create a binary vector of size $k$ with $r$ $1$s. Since the arms in $\mathcal{K}_1$ are identical, this vector represents the outcome of any individual arm in the subset, up to some permutation between the arms. Then, select random $s$ entries from the vector and return their sum. This procedure simulates exactly the distribution of $s$ arms in $\mathcal{K}_1$ given that the sum of $\mathcal{K}_1$ is $r$.
\end{proof}

We now bound the sample complexity for $k\leq\frac{n}{2}$.

\begin{lemma_thm}\label{lemma:lowerBound2}
For any $n$ and $k\leq\frac{n}{2}$, and for any $\epsilon>0$, there exists a reward distribution such that the sample complexity of any $(\epsilon,\delta)$-PAC algorithm is at least
\[M=\Omega\Big(\frac{k}{\epsilon^2}\Big)\]
\end{lemma_thm}
\begin{proof}
We begin with $k=\frac{n}{2}$. Consider problem (\ref{dfn:probLowerBoundk2}), and we show that any algorithm have to use at least $T\geq\frac{ck}{\epsilon^2}$ samples (for some constant $c>0$ to be set later) in order to identify the correct subset with high probability. Assume by contradiction that there is an algorithm that uses $T\leq\frac{ck}{\epsilon^2}$ samples and returns a subset $S_T$ such that
\begin{equation}\label{eqn:contradiction}
    \forall i=1,2.\,P_i[S_T=\mathcal{K}_i]=Pr[S_T=\mathcal{K}_i|\mathcal{I}_i]\geq \frac{3}{4}
\end{equation}
By Lemma \ref{lemma:simuAlg}, it is enough to consider only algorithms that sample only $\mathcal{K}_1$ and $\mathcal{K}_2$. Let $\Omega=\{0,1\}^{n\times T}$ be the sample space of possible rewards of the arms and let $A=\{\omega\subseteq\Omega\,|\,S_T=\mathcal{K}_1\}$ be the event that the algorithm outputs $\mathcal{K}_1$. According to Pinsker's inequality,
\[2(P_1[A]-P_2[A])^2\leq KL(P_1,P_2)=\sum_{t=1}^T\sum_{i=1}^2 KL(P_1^{i,t},P_2^{i,t})\]
where $KL$ is the Kullback-Leibler divergence between two distributions, and $P_j^{i,t}$ denotes the distribution of rewards at time $t$ given that subset $\mathcal{K}_i$ was selected and the profile is $\mathcal{I}_j$. Note that $P_j^{i,t}$ is a binomial distribution with $k$ samples and probability $\frac{1}{2}\pm\frac{\epsilon}{k}$ and thus the $KL$ divergence satisfies
\[KL(P_1^{i,t},P_2^{i,t})=k\cdot KL(\frac{1}{2}+\frac{\epsilon}{k},\frac{1}{2}-\frac{\epsilon}{k})\leq 8k\big(\frac{\epsilon}{k}\big)^2\] 
Therefore we have
\[2(P_1[A]-P_2[A])^2\leq \sum_{t=1}^T\sum_{i=1}^2 KL(P_1^{i,t},P_2^{i,t})\leq T\cdot 8k \frac{\epsilon^2}{k^2}\leq 8c\]
where we used the assumption $T\leq\frac{ck}{\epsilon^2}$. Thus for $c\leq\frac{1}{16}$ we have that $|P_1[A]-P_2[A]|\leq\frac{1}{2}$. Due to assumption (\ref{eqn:contradiction}) we have that $P_2[A]=Pr[S_T=\mathcal{K}_1]|\mathcal{I}_2]\leq\frac{1}{4}$ and therefore
\[P_1[A]\leq P_2[A]+\frac{1}{2}\leq\frac{3}{4}\]
in contradiction to (\ref{eqn:contradiction}). Finally, for $k<\frac{n}{2}$ we may add to the profiles $\mathcal{I}_1,\mathcal{I}_2$ arms with mean $0$ that may only increase the number of samples.
\end{proof}

\subsubsection*{Lower bound for $k\leq\frac{n}{24}$}
To prove the lower bound for $k\leq\frac{n}{24}$, we use Lemma 4 from \cite{audibert2013regret}. For convenience, we cite the lemma.
\begin{lemma_thm}\label{lemma:kl}
Let $l$ and $k$ be integers with $\frac{1}{2}\leq\frac{k}{2} \leq l \leq k$. Let $p,p',q,p_1,\dots,p_{k-1}\in(0, 1)$ with $q \in \{p, p'\}$, $p_1 = \dots = p_l = q$ and $p_{l+1} =\dots = p_{k-1}$. Let $\mathcal{B}$ (resp. $\mathcal{B'}$) be the sum of $k$ independent Bernoulli distributions with parameters $p, p_1, \dots, p_{k-1}$ (resp. $p', p_1, \dots , p_{k-1}$). We have
\[KL(\mathcal{B},\mathcal{B}')\leq\frac{2(p'-p)^2}{(1-p')(k+1)q}\]
\end{lemma_thm}

We now prove the lower bound for $k\leq\frac{n}{24}$.
\begin{lemma_thm}\label{lemma:lowerBound24}
For any $n$ and $k\leq\frac{n}{24}$, and for any $\epsilon>0$, there exists a reward distribution such that the sample complexity of any $(\epsilon,\delta)$-PAC algorithm is at least
\[M=\Omega\Big(\frac{n}{\epsilon^2}\Big)\]
\end{lemma_thm}
\begin{proof}
For any $j\in[n]$ define the following profile
\[\mathcal{I}_j=
\begin{cases}
X_i\sim Ber\Big(\frac{1}{2}\Big) & i\neq j\\
X_i\sim Ber\Big(\frac{1}{2}+\epsilon\Big) & i= j
\end{cases}\]
and also define $\mathcal{I}_0=\{X_i\sim Ber\Big(\frac{1}{2}\Big)\,|\,i=1\dots n\}$. We use the abbreviation $P_j[\cdot]$ ($\mathbb{E}_j[\cdot]$) to denote the probability (expectation) when the arms are distributed according to $\mathcal{I}_j$. Suppose that there exists an algorithm that runs for $T\leq \frac{cn}{\epsilon^2}$ steps for some $c>0$ under profile $\mathcal{I}_0$ and returns a subset $S_T$. We first show that there are many arms that are sampled only a few times and are not part of $S_T$ with high probability.\\
For any $j\in[n]$ let $T_j$ denote the number of times $j$ is sampled. Then,
\[\sum_{j=1}^n\mathbb{E}_0[T_j]=kT\leq\frac{cnk}{\epsilon^2}\]
Then for at least $\frac{2}{3}$ of the arms it holds $\mathbb{E}_0[T_j]\leq\frac{3ck}{\epsilon^2}$ (otherwise the sum over all arms is bigger then $kT$). Accordingly, by Markov inequality for each of these arms
$P_0[T_j\geq T^*]\leq\frac{1}{8}$ where $T^*=\frac{24ck}{\epsilon^2}$.
For similar considerations, for at least $\frac{2}{3}$ of the arms it holds that $P_0[j\in S_T]\leq\frac{3k}{n}\leq\frac{1}{8}$ (where we assumed $k\leq \frac{n}{24}$). Thus, by pigeon hole there exists a subset of arms $B\subset [n]$ such that $|B|\geq\frac{1}{3}n$ and for all $j\in B$ the following holds
\begin{equation}\label{eqn:BsetProbs}
    P_0[T_j>T^*]\leq\frac{1}{8}\; and\;P_0[j\in S_T]\leq\frac{1}{8}
\end{equation}

Fix an arm $j\in B$ and we prove $P_j[j\in S_T]\leq\frac{1}{2}$. Let $\Omega^*$ denote the sample set of possible arms rewards under the restriction that $j$ was sampled at most $T^*$ times, and let $P^*$ denote the corresponding distribution. By Pinsker's inequality, for any event $A\subset\Omega^*$ the distance between two probability distributions satisfy
\begin{equation}\label{pinsker}
    2(P^*_0[A]-P^*_j[A])^2\leq KL(P^*_0,P^*_j)=\sum_{t=1}^T KL(P^{S_t}_0,P^{S_t}_j)
\end{equation}
where $P^{S_t}_j$ denotes the reward distribution of the subset $S_t$ under profile $\mathcal{I}_j$. Note that all arms except $j$ are identically distributed under $\mathcal{I}_0$ and $\mathcal{I}_j$, and therefore for any $S_t$ that does not include $j$ the $KL$ divergence is zero. Hence, we only need to consider rounds $t\in[T]$ when $j$ was sampled as part of $S_t$. By Lemma \ref{lemma:kl} with $p=\frac{1}{2}+\epsilon$ and $p'=q=p_1=\dots=p_{k-1}=\frac{1}{2}$ we have
\[KL(P^{S_t}_0,P^{S_t}_j)\leq\frac{2\epsilon^2}{\frac{1}{2}\cdot\frac{1}{2}(k+1)}\leq\frac{8\epsilon^2}{k}\]

Substituting in (\ref{pinsker}) gives
\[2(P^*_j[A]-P^*_0[A])^2\leq \sum_{t:\,j\in S_t} KL(P^{S_t}_j,P^{S_t}_0)=\sum_{t:\,j\in S_t} \frac{8\epsilon^2}{2k} \leq \frac{24ck}{\epsilon^2} \frac{8\epsilon^2}{2k} =96c\leq\frac{1}{32}\]
where we assumed $c<\frac{1}{3072}$. We conclude that for any event $A\subset\Omega^*$, $P^*_j[A]\leq P^*_0[A]+\frac{1}{8}$. 

Define the following events
\[A=\{j\in S_T \wedge T_j\leq T^*\}\; and\; A'=\{T_j> T^*\}\]
Note that both $A,A'\subset\Omega^*$ since whether $j$ is sampled more than $T^*$ times is completely determined by the first $T^*$ samples. Thus,
\begin{equation*}
    \begin{split}
    P^*_j[A]\leq P^*_0[A]+\frac{1}{8}\leq \frac{1}{8}+\frac{1}{8}=\frac{1}{4}\\
    P^*_j[A']\leq P^*_0[A']+\frac{1}{8}\leq \frac{1}{8}+\frac{1}{8}=\frac{1}{4}
\end{split}
\end{equation*}
where the probabilities are bounded due to (\ref{eqn:BsetProbs}). Finally we have
\[P_j[j\in S_T]\leq P_j[j \in S_T \wedge T\leq T^*]+P_j[T>T^*]\leq\frac{1}{4}+\frac{1}{4}=\frac{1}{2}\]
Namely, every algorithm that runs less then $\frac{cn}{\epsilon^2}$ rounds will err on more than $\frac{1}{3}$ of the instances and return an $\epsilon$-far set with probability at least $\frac{1}{2}$.
\end{proof}

\subsubsection*{Sum up}
In Lemma \ref{lemma:lowerBound24} we showed that for $k\leq\frac{n}{24}$ the sample complexity is at least $\Omega\big(\frac{n}{\epsilon^2}\big)$, and in Lemma \ref{lemma:lowerBound2} we showed that for $k\leq\frac{n}{2}$ it is at least $\Omega\big(\frac{k}{\epsilon^2}\big)$. Note that for $\frac{n}{24}\leq k\leq\frac{n}{2}$ we can write $k=O(n)$ and thus we can sum up both cases to deduce Theorem \ref{thm:lowerBound}.

\section{Sample Complexity in \cite{kuroki2019polynomial}}\label{appendix:kuroki}
We refer to the sample complexity upper bound of $O\big(\frac{\rho(p)}{\Delta^2}k^5n^\frac{1}{4}\log\frac{n}{\delta}\big)$ in \cite{kuroki2019polynomial}, where $\rho(p)$ depends on the distribution of arms' selection $p$. We show that for any $p$, $\rho(p)\geq\frac{n}{k}$, and thus the sample complexity is given by $O\big(\frac{n^\frac{5}{4}k^4}{\Delta^2}\log\frac{n}{\delta}\big)$. 

We start by citing some definitions. For any set $S$ of $k$ arms let $\chi_S\in\{0,1\}^n$ denote its indicator vector. Fix an algorithm for finding the best subset, and let $p(S)$ be the probability that the algorithm selects $S$. Define $\Lambda_p=\sum_{S\subset[n]}p(S)\chi_S\chi_S^\intercal$ and $\rho(p)=\max_S \chi_S^\intercal\Lambda_p^{-1}\chi_S$.
We want to bound $\rho(p)$. For that, we first prove the following claim which will be useful for bounding $\rho(p)$.

% \begin{lemma2}\label{lemma:trace}
% For any invertible matrix $A$ of size $n$, $tr(A)tr(A^{-1})\geq n^2$
% \end{lemma2}
% \begin{proof}
% Let $\lambda_1\dots\lambda_n$ be $A$'s eigenvalues. Then $tr(A)=\sum_{i=1}^n \lambda_i$ and $tr(A^{-1})=\sum_{i=1}^n \lambda_i^{-1}$. According to the inequality of arithmetic and harmonic means we have
% \[\frac{1}{n}\sum_{i=1}^n \lambda_i \geq \frac{n}{\sum_{i=1}^n \lambda_i^{-1}}\]
% which proves the claim.
% \end{proof}

\begin{claim}\label{claim:bilinear}
For any vector $x\in\mathbb{R}^n$ and any invertible and symmetric matrix $A$ of size $n$,
\[(x^\intercal Ax)(x^\intercal A^{-1}x)\geq \norm{x}_2^4\]
\end{claim}
\begin{proof}
Let $v_1\dots v_n$ be $A$'s eigenvectors corresponding to the eigenvalues $\lambda_1\dots\lambda_n$. We write $x=\sum_{i=1}^n \alpha_i v_i$, then 
$x^\intercal Ax=\sum_{i=1}^n \alpha_i^2 \lambda_i$ and $x^\intercal A^{-1}x=\sum_{i=1}^n \alpha_i^2 \lambda_i^{-1}$
since $A$ is symmetric and therefore its eigenvectors are orthonormal. According to the weighted version of the inequality of arithmetic and harmonic means, we have
\[\frac{x^\intercal Ax}{\norm{x}_2^2}= \frac{\sum_{i=1}^n \alpha_i^2 \lambda_i}{\sum_{i=1}^n \alpha_i^2}\geq\frac{\sum_{i=1}^n \alpha_i^2}{\sum_{i=1}^n \alpha_i^2 \lambda_i^{-1}}=\frac{\norm{x}_2^2}{x^\intercal A^{-1}x}\]
\end{proof}

\begin{claim}\label{claim:kuroki}
For any distribution $p$, $\rho(p)\geq\frac{n}{k}$
\end{claim}

\begin{proof}
First consider $\Lambda_p$'s trace.
\[tr(\Lambda_p)=\sum_{S\subset[n]}p(S)tr(\chi_S\chi_S^\intercal)=\sum_{S\subset[n]}p(S)tr(\chi_S^\intercal\chi_S)=\sum_{S\subset[n]}p(S)k=k\]
where we used the fact that $\chi_S$ contains exactly $k$ ones and that $\sum_{S\subset[n]}p(S)=1$. 
%Hence, according to Lemma \ref{lemma:trace}, $tr(\Lambda_p^{-1})\geq\frac{n^2}{tr(\Lambda_p)}=\frac{n^2}{k}$.\\

Next, note that entry $i,j$ in $\Lambda_p$ is the marginal probability $p_{ij}$ of arms $i,j$ being selected together according to $p$. Accordingly, the entries on the diagonal $p_{ii}$ are the marginal probabilities of single arms. We saw that $tr(\Lambda_p)=\sum_{i=1}^n p_{ii}=k$, namely the average $\frac{1}{n}\sum_{i=1}^n p_{ii}=\frac{k}{n}$. Assume that the arms are ordered such that $p_{11}\leq\dots\leq p_{nn}$. Then the average of the minimal $k$ arms satisfies
% \begin{equation}\label{eq:minKavg}
    $\frac{1}{k}\sum_{i=1}^k p_{ii}\leq\frac{k}{n}$.
% \end{equation}
Thus, for the set $S=\{1\dots k\}$, \[\chi_S^\intercal\Lambda_p\chi_S=\sum_{i=1}^k\sum_{j=1}^k p_{ij}\leq k\sum_{i=1}^k p_{ij} \leq k\frac{k^2}{n} \]
where the first inequality is because $\forall i,j, p_{ii}\geq p_{ij}$.\\%, and the second is due to (\ref{eq:minKavg}).\\
Finally, note that $\Lambda_p$ is symmetrical and thus by Claim \ref{claim:bilinear} we have
\[\chi_S^\intercal \Lambda_p^{-1} \chi_S
\geq \frac{\norm{\chi_S}_2^4}{\chi_S^\intercal \Lambda_p \chi_S}\geq\frac{k^2}{\frac{k^3}{n}}\]
which shows that $\rho(p)=\max_{S'}\chi_{S'}^\intercal \Lambda_p^{-1} \chi_{S'}\geq \chi_{S}^\intercal \Lambda_p^{-1} \chi_{S}\geq\frac{n}{k}$.
\end{proof}

\section{Experiments}\label{sec:experiments}

We compared our algorithm to other methods also experimentally, on simulated data. We conducted two experiments, one for the sample complexity and one for the regret.

\subsection{Sample Complexity}
For the sample complexity, we evaluate the accuracy of different sampling methods in comparison to EST1.
%We mainly concetratd on the selection of the matrix to perform the estimation.
%that uses Hadamard matrices. 
Figure \ref{fig:sample}(a) shows the mean square error of EST1, which uses Hadamard matrix, along with two other sampling methods. The first is Leave One Out (LOO), that partitions the arms into sets of size $k+1$ and in each one samples all the $k+1$ subsets of size $k$. The second method samples a random $2k\times2k$ matrix, such that in each row $k$ entries are $+1$ and $k$ are $-1$. In this experiment, each arm is a normal random variable with random mean in $[0,1]$ and $\sigma^2=1$, and we set $n=144$ and $k=8$. The plot shows the average and standard deviation of $1000$ runs. It can be seen that Hadamard significantly outperforms the other two methods.

% \begin{figure}[h]
%     \centering
%     {\includegraphics[width=6.5cm]{images/mse_plot.jpg} }%
%     \caption{Comparing sampling schemes (Hadamard, LOO and random).}%
%     \label{fig:sample}%
% \end{figure}

The high variance in the random matrices' MSE stems from the variance involved in the choice of the matrices. To see that, we tested the relation between the MSE and the condition number of different matrices. In linear regression, the condition number is defined by the ratio between the biggest and the smallest singular values $\frac{\sigma_{max}}{\sigma_{min}}$ and it measures the affect of deviations in the response variable on the estimation error to the. Figure \ref{fig:sample}(b) shows the MSE of $1000$ random matrices as a function of their condition number, where each point is the average of $100$ independent experiments. It can be seen that the MSE is indeed monotone with the condition number, with Spearman correlation of 96\%. However, we note that while the relation between the two is expected to be linear according to theory, the relation observed in the experiments is not linear. 
%Trying to fit the data to the form $y=ax^c$ yields $c\approx 0.5$ with $r^2=0.89$, namely the relation is approximately squared.

% \begin{figure}[h]
%     \centering
%     {\includegraphics[width=6.5cm]{images/mse_cond.jpg} }%
%     \caption{MSE of $1000$ random matrices as a function of their condition number.}%
%     \label{fig:cond}%
% \end{figure}

\begin{figure}[h]
    \centering
    {{\includegraphics[width=6.5cm]{images/mse_plot.jpg} }}%
    \qquad
    {{\includegraphics[width=6.5cm]{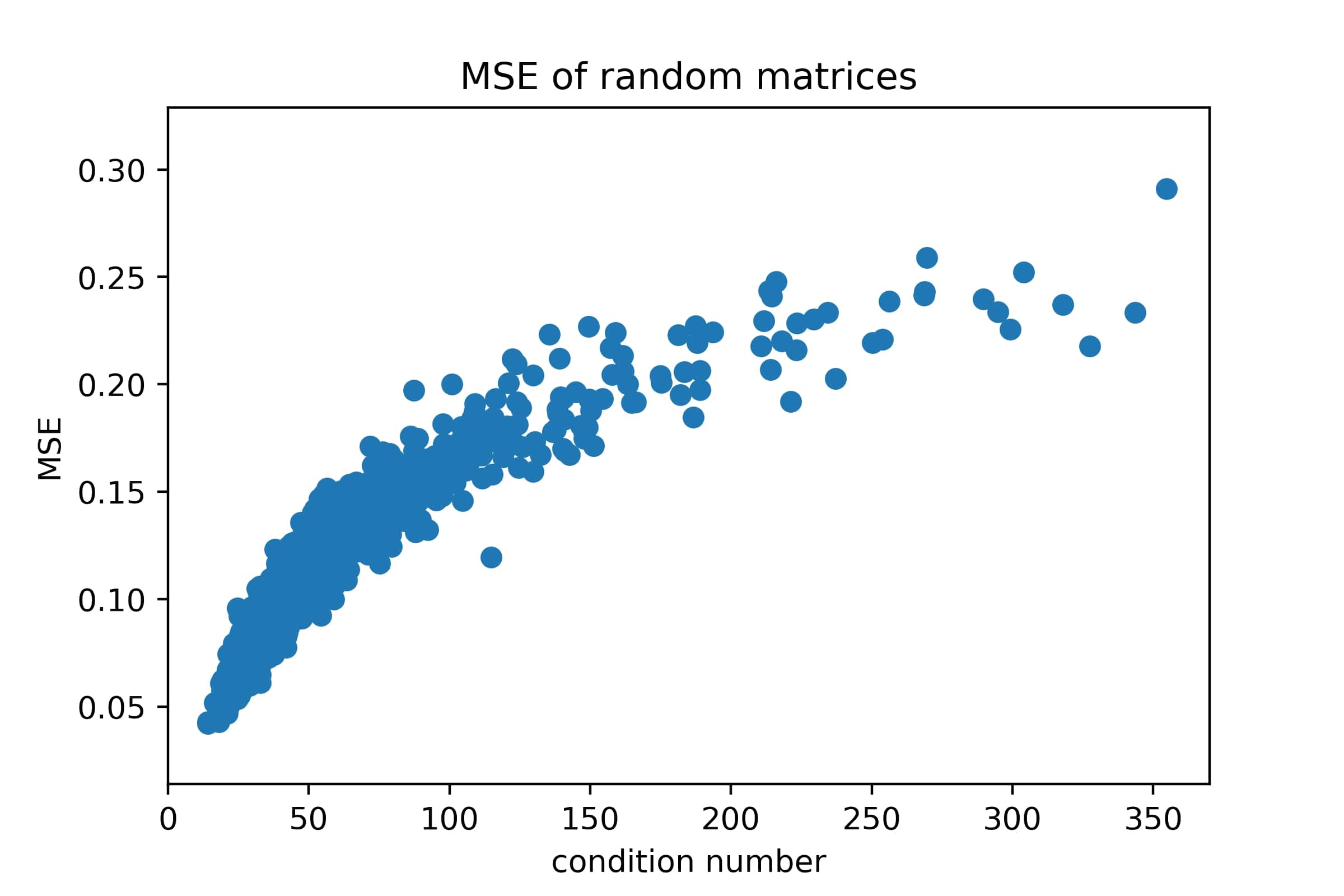} }}%
    \caption{(a) Comparing sampling schemes (Hadamard, LOO and random).
    (b) MSE of random matrices as a function of their condition number.}%
    \label{fig:sample}%
\end{figure}

\newpage
\subsection{Regret}
For the regret, we compared CSAR's performance to the Sort \& Merge algorithm in \cite{agarwal2018regret}. Figure \ref{fig:regret}(a) shows the cumulative regret as a function of time for both algorithms. In this experiment, we initialized the arms to be Bernoulli random variables with random mean in $[0,1]$, and we set $n=24$ and $k=2$. The plot shows the average and standard deviation of $100$ runs. It can be seen that CSAR achieves significantly lower regret than Sort \& Merge. 
In addition, we test the consistency of this gap for different $k$s. Figure \ref{fig:regret}(b) shows the cumulative regret after 5 millions steps for differnt $k$ values. The plot shows the average and standard deviation of $35$ runs.
Note that the large deviations in Sort \& Merge's regret in both plots result from the random initialization of the arms that might effect the exploration's duration dramatically.

\begin{figure}[h]
    \centering
    {{\includegraphics[width=6.5cm]{images/regret_plot.jpg} }}%
    \qquad
    {{\includegraphics[width=6.5cm]{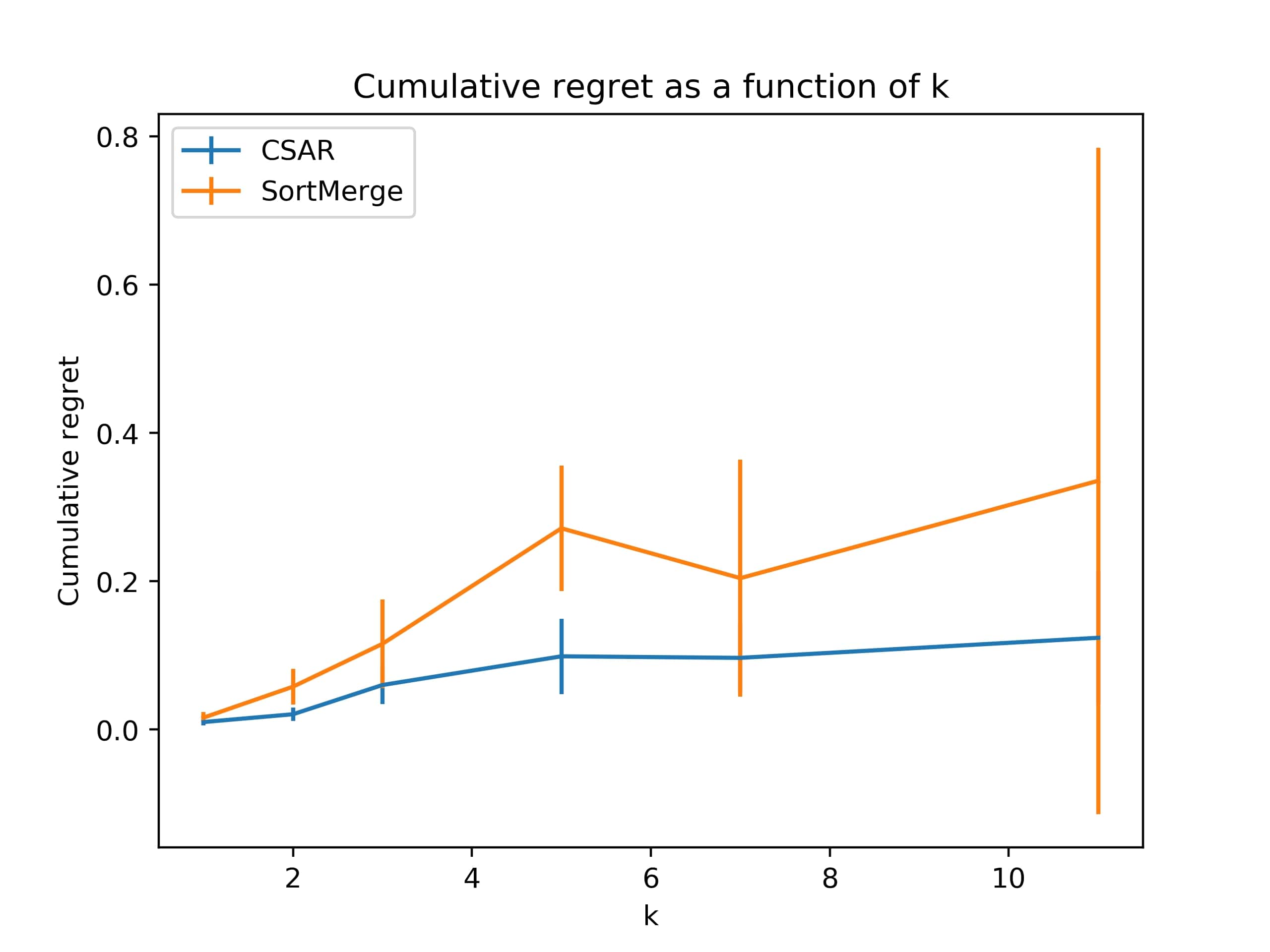} }}%
    \caption{Comparing regret (CSAR vs Sort \& Merge) as a function of: (a) time horizon $T$ (b) subset size $k$}%
    \label{fig:regret}%
\end{figure}

\end{document}